\documentclass{article}

\usepackage[preprint]{neurips_2026}

\usepackage[utf8]{inputenc}
\usepackage[T1]{fontenc}
\usepackage{hyperref}
\usepackage{url}
\usepackage{booktabs}
\usepackage{amsfonts}
\usepackage{amsmath}
\usepackage{amssymb}
\usepackage{nicefrac}
\usepackage{microtype}
\usepackage{xcolor}
\usepackage{graphicx}
\usepackage{algorithm}
\usepackage{algorithmic}
\usepackage{multirow}
\usepackage{subcaption}
\usepackage{amsthm}
\usepackage{enumitem}
\usepackage{svg}
\usepackage{caption}

\newtheorem{theorem}{Theorem}

\newtheorem{corollary}[theorem]{Corollary}

\newtheorem{assumption}{Assumption}
\newtheorem{remark}{Remark}
\theoremstyle{definition}

\title{SEAD: Competence-Aware On-Policy Distillation via Entropy-Guided Supervision}

\author{
  \textbf{Chia-Hsuan Lee, Zelei Cheng, Yu Wang, Renkun Ni} \\
  \textbf{Sambit Sahu, Shi-Xiong Zhang, William Campbell} \\
  Capital One
}

\begin{document}

\maketitle
\begin{abstract}
  On-policy distillation (OPD) has a property absent in offline distillation and RL:
  \emph{teacher supervision quality depends on student competence}. Incoherent rollouts yield    
  noisy gradients; already-mastered tokens yield redundant ones. This creates waste at three scales---tokens, training phases, and prompts---yet existing methods supervise uniformly. We introduce \textbf{SEAD},     
  which uses entropy as a unified probe of this competence-dependent degradation at three scales:
   (1)~joint teacher--student entropy partitions tokens into zones receiving tailored divergences
   or zero gradient (${\sim}50\%$ skipped); (2)~a cosine schedule anneals from forward to reverse
   KL as competence grows; (3)~a competence-gated curriculum introduces prompts
  easy-to-hard. These components are symbiotically necessary: token selection
   requires coherent rollouts (curriculum), annealing requires monotonic improvement (also
  curriculum). On OLMo-3 (7B$\to$32B), SEAD achieves +4.8 avg accuracy over vanilla OPD across six math benchmarks, with
   ablations confirming super-additive interactions.
  \end{abstract}


\section{Introduction}
\label{sec:intro}

Large reasoning models (49B+) achieve strong performance but are prohibitively expensive for deployment. Knowledge distillation compresses these capabilities into smaller students. The dominant \emph{off-policy} paradigm---training on static teacher-generated traces---suffers from exposure bias: prediction errors compound autoregressively at inference \citep{agarwal2024gkd, opd_survey_2026}.

On-policy distillation (OPD) addresses this by training on student-generated rollouts scored by the teacher \citep{agarwal2024gkd, lu2025thinking_machines}, matching or exceeding RL methods like GRPO \citep{reopold2026, eopd2026, gopd2026}. However, existing OPD methods apply supervision \emph{uniformly}---the same divergence, on every token, at every phase, for every prompt.

We identify a single structural problem underlying this uniformity: \textbf{in OPD, supervision quality depends on student competence}. Unlike off-policy KD (teacher traces are always coherent) or RL (binary rewards are equally reliable), OPD's per-token teacher corrections are only informative when the student's rollouts are sufficiently coherent. This competence-dependent degradation manifests at three scales:
\begin{itemize}[leftmargin=1.5em, itemsep=1pt]
\item \textbf{Token level:} ${\sim}50\%$ of tokens are deterministic for both models---supervising them wastes compute. Among the remaining, some call for sharpening (RKL) while others require diversity preservation (FKL) \citep{eopd2026}.
\item \textbf{Temporal level:} The optimal divergence evolves from mode-covering (early) to mode-seeking (late), yet prior solutions use manual two-stage switches \citep{reopold2026, paced2026}.
\item \textbf{Prompt level:} Problems beyond the student's capability produce incoherent rollouts on which teacher supervision is noise \citep{rethinking_opd2026, opsd2026}.
\end{itemize}
These are not three independent problems but three symptoms of \emph{ignoring how supervision quality varies with competence}. We observe that \emph{entropy} provides a unified observable for this quantity across all three scales.

We introduce \textbf{SEAD}, a framework operationalizing this principle:
\begin{enumerate}[itemsep=1pt]
    \item \textbf{Token-level: Sparse Entropy-Adaptive Divergence.} Joint teacher-student entropy partitions tokens into Zone~A (skip), Zone~B (RKL), Zone~C (FKL)---jointly determining selection \emph{and} divergence type.
    \item \textbf{Temporal: Competence-Driven Annealing.} Continuous FKL$\to$RKL transition tracking the evolving active token composition.
    \item \textbf{Prompt-level: Competence-Gated Curriculum.} First prompt-level curriculum for OPD, addressing an open problem flagged by \citet{opsd2026, rethinking_opd2026}.
\end{enumerate}
Crucially, these are symbiotically necessary: token selection requires coherent rollouts (curriculum), and annealing requires monotonic competence growth (also curriculum). The ablation (Section~\ref{sec:analysis}) confirms super-additive interactions.

We validate SEAD on OLMo-3 (7B$\to$32B) and Nemotron (8B$\to$49B) across MATH-500,
Minerva-Math, AIME~2024/2025, AMC~2023, and OlympiadBench. The full framework achieves +4.8
average over vanilla OPD.

\begin{figure}[t]
\centering
\includegraphics[width=\textwidth]{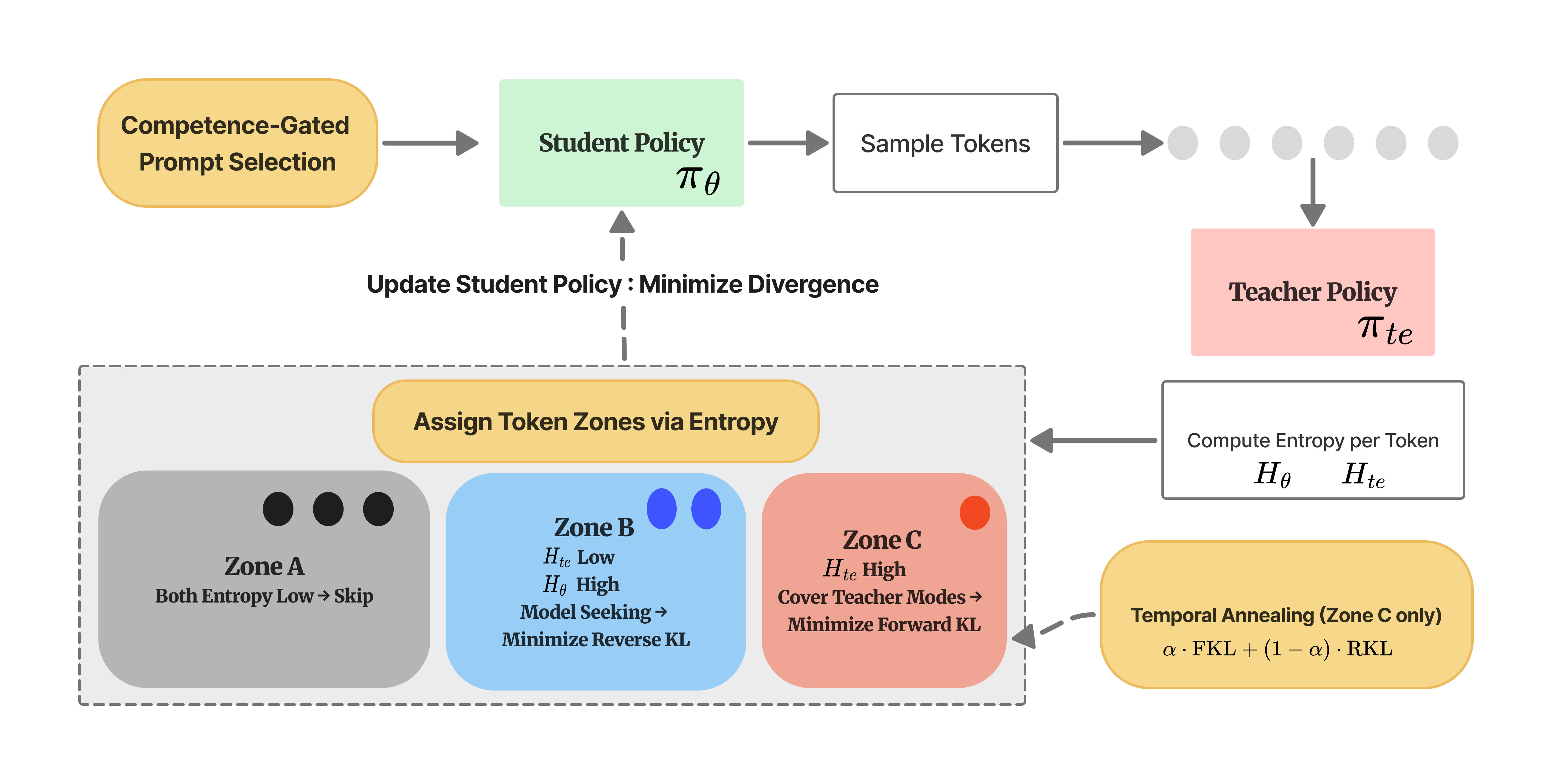}
  \caption{\textbf{Overview of SEAD.} First, the    
  \emph{competence-gated curriculum} (left, yellow) selects prompts within the student's current
  ability frontier ($d_i \leq c(t)$), ensuring rollouts are coherent enough for meaningful
  supervision. The student $\pi_\theta$ generates a rollout, which the teacher $\pi_{\text{te}}$
  scores with per-token logits. From these, we compute joint entropy ($H_\theta$,
  $H_{\text{te}}$) and partition tokens into three zones: \textbf{Zone~A} (gray)---both models
  confident, zero gradient (${\sim}50\%$ of tokens); \textbf{Zone~B} (blue)---teacher confident
  but student uncertain, supervised via reverse KL to sharpen toward the teacher's mode;
  \textbf{Zone~C} (red)---teacher uncertain at reasoning forks, supervised via forward KL to
  preserve multi-path diversity. A \emph{temporal annealing} schedule (bottom-right) modulates
  Zone~C as $\alpha \cdot \mathcal{L}_{\text{FKL}} + (1-\alpha) \cdot \mathcal{L}_{\text{RKL}}$,
  smoothly transitioning from exploration ($\alpha{=}0.8$) to refinement ($\alpha{=}0$) over
  training. The combined loss updates the student. See
  Algorithm~\ref{alg:sead} for the full procedure.}
\label{fig:sead_framework}
\end{figure}

\section{Method}
\label{sec:method}
\subsection{Preliminaries and Problem Setting}
\label{sec:prelim}

Let $\pi_\theta$ and $\pi_{\text{te}}$ denote the student's and teacher's policies respectively. For a given prompt $\mathbf{q}$, let $\mathbf{c}_t = (\mathbf{q}, x_1, \dots, x_{t-1})$ denote the context up to step $t$, and $x_t \in \mathcal{V}$ the generated token. Each OPD iteration: (1)~samples rollouts $\mathbf{x} \sim \pi_{\theta_{\text{old}}}(\cdot | \mathbf{q})$; (2)~queries teacher logits on the rollout; (3)~updates $\pi_\theta$ via a divergence loss. The two standard choices are \textbf{forward KL} (mode-covering):
\begin{equation}
\mathcal{L}_{\text{FKL}}(t) = \sum_{v \in \mathcal{V}} \pi_{\text{te}}(v | \mathbf{c}_t) \log \frac{\pi_{\text{te}}(v | \mathbf{c}_t)}{\pi_\theta(v | \mathbf{c}_t)}
\label{eq:fkl}
\end{equation}
and \textbf{reverse KL} (mode-seeking):
\begin{equation}
\mathcal{L}_{\text{RKL}}(t) = 
\sum_{v \in \mathcal{V}} \pi_\theta(v | \mathbf{c}_t) \log \frac{\pi_\theta(v | \mathbf{c}_t)}{\pi_{\text{te}}(v | \mathbf{c}_t)}
\label{eq:rkl}
\end{equation}
Forward KL encourages the student to cover all modes of the teacher, preserving diversity at reasoning branch points. Reverse KL drives the student to concentrate on the teacher's high-probability modes, yielding sharper outputs but risking premature entropy collapse~\citep{reopold2026}.

\begin{assumption}[OPD supervision quality --- modulated PL condition]\label{as:quality}
Let $\ell_i(\theta)$  denote the per-prompt OPD loss and $\ell_i^\star = \inf_\theta \ell_i(\theta)$.
There exists a non-decreasing function $\phi\colon [0,1] \to [0,1]$ with $\phi(0)=0$, $\phi(1)=1$, and a constant $\mu > 0$ such that for each prompt~$i$:  :
\begin{equation}
\|\nabla \ell_i(\theta)\|^2 \geq \phi(p_i(\theta))\,\mu\,(\ell_i(\theta) - \ell_i^\star).
\label{eq:mod-pl}
\end{equation}
\end{assumption}

\noindent Intuitively, gradient signal scales with competence: when the student cannot solve a problem at all ($p_i \approx 0$), its rollouts are incoherent and $\phi(p_i) \approx 0$, so the lower bound is almost a zero regardless of the teacher's quality. Entropy serves as the observable proxy for this principle: joint teacher--student entropy reveals whether a token position is informative, redundant, or noise. Assumption~\ref{as:quality} motivates a unified principle: \emph{allocate compute only where supervision quality is high}. SEAD operationalizes this at three granularities: (i)~\emph{token-level}---skip tokens where both teacher and student are already confident (zero information gain); (ii)~\emph{temporal}---shift from exploratory FKL to sharpening RKL as competence grows; (iii)~\emph{prompt-level}---restrict training to prompts where rollouts are coherent enough for the teacher to provide meaningful corrections. We present these in order of increasing scope, noting that the prompt-level curriculum (Sec.~\ref{sec:curriculum}) is the foundational enabler: it ensures the entropy landscape is well-behaved, which the token-level selector requires, and it guarantees monotonic competence growth, which the temporal annealer assumes.

\subsection{The Unified SEAD Objective}
\label{sec:combined}

SEAD integrates token-level selection, temporal annealing, and prompt-level curriculum into a single loss. Let $\mathcal{B}$ and $\mathcal{C}$ denote the sets of token indices assigned to RKL and FKL respectively (formally defined in Sec.~\ref{sec:sead}), and let $\lambda > 0$ be a hyperparameter balancing the intrinsic scale differences between the two divergence terms:
\begin{equation}
\mathcal{L}_{\text{SEAD}}(\theta, t_{\text{step}}) = \frac{1}{|\mathcal{B} \cup \mathcal{C}|} \left( \sum_{t \in \mathcal{B}} \mathcal{L}_{\text{RKL}}(t) + \alpha(t_{\text{step}}) \cdot \lambda \sum_{t \in \mathcal{C}} \mathcal{L}_{\text{FKL}}(t) \right), \quad \mathbf{q} \sim \text{Uniform}(\mathcal{D}(t_{\text{step}}))
\label{eq:sead_unified}
\end{equation}
Three control variables---all functions of training progress, all serving the same principle (allocate compute only where supervision quality is high):
\begin{itemize}[leftmargin=1.5em, itemsep=0pt]
    \item \textbf{Zone partition} $\{\mathcal{A}, \mathcal{B}, \mathcal{C}\}$: which tokens receive gradient and which divergence (Sec.~\ref{sec:sead});
    \item \textbf{Annealing coefficient} $\alpha(t_{\text{step}})$: FKL/RKL balance evolving with competence (Sec.~\ref{sec:anneal});
    \item \textbf{Eligible set} $\mathcal{D}(t_{\text{step}}) = \{\mathbf{q}_i : d_i \leq c(t_{\text{step}})\}$: prompts within the competence frontier (Sec.~\ref{sec:curriculum}).
\end{itemize}

\begin{algorithm}[t]
\caption{SEAD: Competence-Aware On-Policy Distillation}
\label{alg:sead}
\begin{algorithmic}[1]
\REQUIRE Student $\pi_\theta$, teacher $\pi_{\text{te}}$, prompts $\mathcal{Q}$ with difficulty scores $\{d_i\}$, zone percentiles $(\rho_A, \rho_B, \rho_C)$, annealing schedule $\alpha(\cdot)$, competence function $c(\cdot)$
\STATE Precompute per-prompt difficulty $d_i = 1 - p_i$ via student pass rate
\FOR{$t_{\text{step}} = 1$ to $T_{\text{total}}$}
    \STATE $\mathcal{D} \leftarrow \{\mathbf{q}_i : d_i \leq c(t_{\text{step}})\}$ \hfill \COMMENT{Competence-gated curriculum}
    \STATE Sample prompt batch $\{\mathbf{q}_i\} \subset \mathcal{D}$; generate rollouts $\mathbf{x}_i \sim \pi_{\theta_{\text{old}}}(\cdot | \mathbf{q}_i)$
    \STATE Query teacher $\pi_{\text{te}}(\cdot | \mathbf{c}_t)$; compute $H_{\text{te}}(t)$, $H_\theta(t)$ per position
    \STATE Assign tokens to Zone A, B, C via percentile thresholds
    \STATE $\alpha \leftarrow \alpha(t_{\text{step}})$ \hfill \COMMENT{Competence-driven annealing}
    \STATE $\mathcal{L} \leftarrow \frac{1}{|\mathcal{B} \cup \mathcal{C}|}\!\left(\sum_{t \in \mathcal{B}} \mathcal{L}_{\text{RKL}}(t) + \alpha \lambda \sum_{t \in \mathcal{C}} \mathcal{L}_{\text{FKL}}(t)\right)$
    \STATE Update $\theta$ via clipped gradient step on $\mathcal{L}$
\ENDFOR
\end{algorithmic}
\end{algorithm}

\begin{remark}[Symbiotic necessity]\label{rem:symbiosis}
The three components are preconditions for each other: \textbf{without curriculum}, incoherent rollouts corrupt the entropy landscape, causing zone misclassification (confirmed in Table~\ref{tab:ablation}); \textbf{without annealing}, a fixed divergence ratio drives early mode collapse or late over-exploration; \textbf{without token selection}, ${\sim}50\%$ of compute is wasted on tokens where supervision is already maximal.
\end{remark}

\subsection{Token-Level: Sparse Entropy-Adaptive Divergence}
\label{sec:sead}

We define token-level entropy $H_{\text{te}}(t) = -\sum_{v} \pi_{\text{te}}(v | \mathbf{c}_t) \log \pi_{\text{te}}(v | \mathbf{c}_t)$ and $H_{\theta}(t)$ analogously. To maintain computational tractability and avoid the prohibitive overhead of a full softmax over the entire vocabulary $\mathcal{V}$ during training, $H_{\text{te}}(t)$ is approximated using the top-$k$ vocabulary subset. SEAD partitions tokens into
\begin{itemize}[itemsep=2pt]
    \item \textbf{Zone A} (Skip, ${\sim}\rho_A\%$): Both $H_{\text{te}}(t)$ and $H_\theta(t)$ low. \textbf{Zero gradient}---the vast majority of tokens (connectives, formatting, deterministic steps) fall here.
    \item \textbf{Zone B} (RKL, ${\sim}\rho_B\%$): $H_{\text{te}}(t)$ low, $H_\theta(t)$ high. Student should \emph{sharpen} toward the confident teacher.
    \item \textbf{Zone C} (FKL, ${\sim}\rho_C\%$): $H_{\text{te}}(t)$ high. Genuine reasoning forks---student should \emph{cover} teacher modes.
\end{itemize}
The per-step loss is:
\begin{equation}
\mathcal{L}_{\text{SEAD}}^{\text{token}} = \frac{1}{|\mathcal{B} \cup \mathcal{C}|} \left( \sum_{t \in \mathcal{B}} \mathcal{L}_{\text{RKL}}(t) + \lambda \sum_{t \in \mathcal{C}} \mathcal{L}_{\text{FKL}}(t) \right)
\label{eq:sead}
\end{equation}
with defaults $\rho_A = 50$, $\rho_B = 40$, $\rho_C = 10$. Unlike EOPD \citep{eopd2026} (teacher entropy only) and TIP \citep{tip2026} (weighting, single divergence), SEAD jointly determines both \emph{selection} and \emph{divergence type} via \emph{joint} teacher-student entropy, with extreme sparsity (${\sim}20\%$ active).

\paragraph{Sparse selection is approximately lossless.}
The following theorem shows that Zone~A tokens contribute negligible gradient, justifying their exclusion. The key precondition is that Zone~A tokens have genuinely low entropy for \emph{both} teacher and student---a property that holds when the curriculum (Sec.~\ref{sec:curriculum}) ensures coherent rollouts, preventing spurious low-entropy assignments from misaligned contexts.

\begin{theorem}\label{thm:sparse-gradient}
Under mild regularity (bounded score functions) and the assumption that Zone~A tokens satisfy $H_{\mathrm{te}}(t) \leq \tau$, $H_\theta(t) \leq \tau$ with $\mathrm{TV}(\pi_{\mathrm{te}}, \pi_\theta) \leq \delta(\tau)$, the full gradient $g$ and sparse gradient $\hat{g}$ (computed on Zones B$\cup$C only) satisfy:
\begin{equation}
  \|g - (1-s)\hat{g}\| \leq s\,G\bigl(\sqrt{2\tau} + 2\,\delta(\tau) + \sqrt{\tau/2}\,\log|\mathcal{V}|\bigr) = O(s\,G\sqrt{\tau}\,\log|\mathcal{V}|)
  \label{eq:sparse-bound}
\end{equation}
where $s = |\mathcal{A}|/N < 1$ is the skip fraction ($N$ = sequence length), $G$ bounds score function norms, and $|\mathcal{V}|$ is the vocabulary size.
\end{theorem}

The bound vanishes as $\tau \to 0$: Zone~A tokens are selected precisely for near-zero entropy, making their gradient contributions negligible. Crucially, this guarantee requires the entropy landscape to be \emph{well-behaved}---which holds when curriculum ensures coherent rollouts.

\begin{remark}[Implementation scaling]\label{rem:scaling}
In practice (Eq.~\ref{eq:sead}, Algorithm~\ref{alg:sead}), we normalize by $|\mathcal{B}\cup\mathcal{C}|$ rather than the full sequence length $N$, implicitly rescaling the gradient by $1/(1-s)$. This is equivalent to using an effective learning rate $\eta_{\mathrm{eff}} = \eta/(1-s)$ on the active tokens, which we absorb into learning rate tuning. The directional guarantee of Theorem~\ref{thm:sparse-gradient} is unaffected: $\cos(g, \hat{g}) \to 1$ as $\tau \to 0$, ensuring the sparse update follows the same descent direction as the dense baseline.
\end{remark}

\paragraph{Connection to temporal scheduling.} As the student improves, Zone~A expands and the active set shifts toward Zone~B. The optimal FKL/RKL ratio among active tokens thus evolves, motivating a temporal schedule.

\subsection{Temporal: Competence-Driven Divergence Annealing}
\label{sec:anneal}

As the student masters more tokens, the active token composition shifts: early in training Zone~C (both uncertain) is large, calling for mode-covering FKL; late in training Zone~B (teacher confident, student uncertain) dominates, calling for mode-seeking RKL. The annealing schedule tracks this:
\begin{equation}
\alpha(t_{\text{step}}) = \alpha_{\text{end}} + \frac{\alpha_{\text{start}} - \alpha_{\text{end}}}{2}\left(1 + \cos\!\left(\frac{t_{\text{step}}}{T_{\text{total}}} \cdot \pi\right)\right)
\label{eq:anneal}
\end{equation}
decreasing from $\alpha_{\text{start}} = 0.8$ to $\alpha_{\text{end}} = 0.0$. Combined with SEAD zones, $\alpha$ modulates Zone~C's FKL weight in Eq.~\ref{eq:sead_unified}, implementing a continuous exploration$\rightarrow$refinement transition that generalizes the discrete two-stage switches of \citet{reopold2026} and \citet{paced2026}.

Temporal annealing assumes monotonically improving competence---an assumption that can be violated under uniform prompt sampling. The following curriculum ensures this precondition holds.

\subsection{Prompt-Level: Competence-Gated Curriculum}
\label{sec:curriculum}

The deepest manifestation of the supervision quality principle (Assumption~\ref{as:quality}): on prompts far beyond the student's capability, $p_i(\theta) \approx 0$ implies $\phi(p_i(\theta)) \approx 0$, so the modulated PL condition yields vanishing gradient signal. The rollouts are incoherent and teacher corrections amount to noise. This degradation is absent in off-policy KD (teacher traces are always coherent, $\phi \equiv 1$) and RL (binary rewards are equally reliable regardless of rollout quality). The curriculum keeps $\phi(p_i(\theta))$ bounded away from zero on all eligible prompts---a precondition for both Theorem~\ref{thm:sparse-gradient} (the entropy landscape must be well-behaved for Zone~A tokens to genuinely have low entropy) and temporal annealing (competence must grow monotonically for the schedule to track the correct FKL/RKL ratio).

\paragraph{Difficulty estimation.}
We sample $K$ rollouts per prompt from $\pi_{\theta_0}$ and compute difficulty $d_i = 1 - p_i$ where $p_i = \frac{1}{K}\sum_k \mathbf{1}[\mathrm{correct}(\mathbf{x}^{(k)}_i)]$. Pass rate is a binary discretization of rollout entropy, maintaining the entropy-as-unified-signal principle.

\paragraph{Competence-based progression.}
Following \citet{platanios2019competence}, a competence function $c(t) = \min(1, c_0(1 + t/T)^p)$ grows monotonically, with eligible set $\mathcal{D}(t) = \{\mathbf{q}_i : d_i \leq c(t)\}$. In OPD, $c(t)$ delineates prompts where rollouts are coherent enough for informative teacher supervision. 

\paragraph{Staleness of difficulty scores.}
Because $d_i$ is computed from the initial student $\pi_{\theta_0}$, the relative ordering of prompts may drift as training progresses. While monotonic improvement guarantees that initially ``easy'' prompts remain solvable, mid-tier prompts may improve at different rates. Consequently, we rely on the curriculum as a coarse filter rather than a strict total ordering, trading precise difficulty tracking for computational efficiency. The primary risk is minor sample inefficiency, not gradient corruption, as the curriculum's only theoretical requirement is maintaining a minimum competence floor $\phi(p_i(\theta)) \geq c > 0$ on all eligible prompts.

Under Assumption~\ref{as:quality}, and assuming $\beta$-smoothness of the loss and bounded gradient variance $\sigma^2$ (see Appendix~\ref{app:proof-curriculum}), we have the following theorem:

\begin{theorem}[Curriculum convergence]\label{thm:curriculum}
The curriculum achieves $\epsilon$-accuracy in $T^{\textup{C}} = \mathcal{O}(\beta\sigma^2/((c\mu)^2\epsilon))$ steps, versus $T^{\textup{U}} \geq \Omega(\beta\sigma^2/((\bar\phi^{\textup{U}}\mu)^2\epsilon))$ for uniform sampling, where $\bar\phi^{\textup{U}} = \frac{1}{n}\sum_i \phi(p_i(\theta))$ is the average modulation under uniform sampling. This yields a theoretical speedup factor of $S^2 = (c/\bar\phi^{\textup{U}})^2$ in the worst case---\textbf{a separation unique to OPD} and absent in off-policy KD ($\phi \equiv 1$) or RL (unbiased gradients regardless of competence).
\end{theorem}

\begin{corollary}[Mutual reinforcement]\label{cor:joint}
The three SEAD components form a virtuous cycle, not a circular dependency. The resolution is \emph{staged bootstrapping}:
\begin{enumerate}[leftmargin=1.5em, itemsep=1pt]
    \item \textbf{Curriculum provides the base.} By restricting to prompts with $\phi(p_i) \geq c > 0$, the curriculum ensures coherent rollouts, which in turn make the entropy landscape reliable (low-entropy tokens in Zone~A genuinely agree between teacher and student).
    \item \textbf{Reliable entropy enables sparse selection.} Given a well-behaved entropy landscape, Theorem~\ref{thm:sparse-gradient} guarantees that skipping Zone~A introduces negligible error, bounded by $\mathcal{O}(s\,G\sqrt{\tau_{\max}})$.
    \item \textbf{Sparse selection + monotonic competence enable annealing.} With competence growing monotonically (Theorem~\ref{thm:curriculum}) and Zone composition shifting predictably, the cosine schedule $\alpha(t)$ correctly tracks the exploration$\to$refinement transition.
\end{enumerate}
As stages progress, the competence lower bound $c$ tightens, Zone~A entropy threshold $\tau$ decreases, and the sparse gradient bound improves: all three components strengthen jointly.
\end{corollary}

\section{Evaluation}
\subsection{Experimental Setup}
\label{sec:setup}
\paragraph{Models.}
We evaluate our approach on two distinct model families to demonstrate generalizability. The first is \textbf{Nemotron} (Nano-8B student, Super-49B teacher)and the second is \textbf{OLMo} (7B-Instruct-SFT student, 32B-Instruct teacher), which allows us to rigorously test our method in the challenging small teacher-student capability gap regime.

\paragraph{Training.}
We trained our models using 4 nodes. Each node is equipped with 8 H100 80GB GPUs and 1TB memory. We utilize a rigorously deduplicated version of the DAPO-Math-17K dataset \citep{dapo2025} for all training phases. More training details could be found in Appendix~\ref{app:train_detail}.

\paragraph{Evaluation.}
Our evaluation suite comprehensively assesses mathematical reasoning capabilities across varying difficulty levels. We report greedy Pass@1 accuracy on standard reasoning benchmarks, including MATH-500, Minerva-Math, and OlympiadBench. To evaluate performance on highly complex, competition-level mathematics, we utilize the AMC~2023 and AIME~2024/2025 datasets, reporting the Pass@32 metric (the average pass rate over 32 independently sampled trajectories per problem).

\paragraph{Baselines.}
We benchmark our approach against several strong paradigms: Group Relative Policy Optimization (GRPO) \citep{shao2024deepseekmath}, Vanilla OPD utilizing full-token Reverse Kullback-Leibler (RKL) divergence~\citep{lu2025thinking_machines}, and On-Policy Self-Distillation (OPSD) \citep{opsd2026}.

\subsection{Main Results}

We present the results of OLMo model pair in Table~\ref{tab:main_olmo} and leave Nemotron results in Appendix~\ref{app:exp_nemotron}.

\begin{table}[t]
\caption{Results on OLMo model pair (OLMo-7B student, OLMo-32B teacher).}
\label{tab:main_olmo}
\centering
\small
\resizebox{\textwidth}{!}{
\begin{tabular}{lccccccc}
\toprule
\textbf{Method} & \textbf{MATH-500} & \textbf{Minerva} & \textbf{Olympiad} & \textbf{AMC23} & \textbf{AIME24} & \textbf{AIME25} & \textbf{Avg.} \\
\midrule
\multicolumn{8}{l}{\textit{Off-the-shelf Models}} \\
\midrule
OLMo-7B-Instruct (Student)  & 87.2 & 41.9 & 58.2 & 80.4 & 46.8 & 34.6 & 58.2 \\
OLMo-32B-Instruct (Teacher) & 94.0 & 62.9 & 68.1 & 97.7 & 69.2 & 60.6 & 75.4 \\
\midrule
\multicolumn{8}{l}{\textit{Baselines}} \\
\midrule
OPSD                         & 86.4 & 43.8 & 57.8 & 80.2 & 44.0 & 36.8 & 58.2 \\
GRPO                         & 84.8   & 43.0   & 58.8   & 81.0   & 45.0   & 35.5   & 58.0   \\
OPD (RKL $k$=1)              & 87.0 & 41.9 & 60.1 & 80.9 & 47.9 & 37.3 & 59.2 \\
\midrule
\multicolumn{8}{l}{\textbf{\textit{Proposed (SEAD components)}}} \\
\midrule 
KL Annealing only 
& 87.8 & 43.4 & 58.5 & 82.9 & 48.2 & 35.6 & 59.4 \\
Token Zones only             & 87.0 & 41.9 & 60.1 & 80.9 & 47.9 & 39.3 & 59.5 \\
Token Zones + KL Annealing & 91.0   & 43.0   & \textbf{62.5}   & 89.1   & \textbf{54.5}   & 42.1   & 63.7   \\
SEAD (full) & \textbf{91.2}   & \textbf{44.5}   & 62.1   & \textbf{89.8}   & 53.9   & \textbf{42.5}   & \textbf{64.0}   \\
\bottomrule
\end{tabular}
}
\end{table}

\paragraph{The Distillation Challenge.}
As shown in Table~\ref{tab:main_olmo}, bridging the massive 17.2-point performance gap between the OLMo teacher and student is challenging. While the teacher establishes a strong upper bound (75.4 average), standard off-policy distillation (OPD) baselines fail to meaningfully improve upon the base student. OPSD and GRPO stagnate near the student's 58.2 average, while Vanilla OPD (RKL $k=1$) yields a marginal $+1.0$ improvement.

\paragraph{Efficacy of Proposed Components.}
KL Annealing alone provides a modest gain (+0.2 avg), confirming that divergence
   scheduling without token-level structure has limited impact. Token Zones alone
  similarly yields marginal improvement (+0.3), as zone-based sparsity without a
  matching temporal schedule leaves the FKL/RKL balance untuned. The critical
  finding is their \emph{synergy}: combining Token Zones with KL Annealing
  produces a +4.5 jump to 63.7 average---a $20\times$ amplification over either
  component in isolation. This validates the symbiotic design: zone sparsity
  concentrates gradient on informative tokens, while annealing ensures the
  divergence applied to those tokens evolves appropriately. Adding the
  competence-gated curriculum (SEAD full) further improves to 64.0, with the
  easy-to-hard ordering providing the coherent rollouts that stabilize
  entropy-based zone assignment.

\paragraph{Combined Method Performance.}
Combining Token Zones with KL Annealing yields robust performance, achieving a 63.7 average (a $+5.5$ absolute increase over the base student), with pronounced gains on the most rigorous datasets, including $+7.7$ points on AIME24 and $+7.5$ points on AIME25. Ultimately, the full SEAD method further pushes the average to 64.0. These results indicate that stabilizing the token zone strategy with an annealed KL penalty effectively mitigates baseline stagnation, drastically elevating student capabilities and closing the teacher-student gap.

\subsection{Ablation Study}
  We conduct a $2^3$ factorial ablation on OLMo-3-7B/32B to isolate three components: \textbf{T}~(Token SEAD, 50/40/10 zone assignment), \textbf{A}~(KL    
  annealing from 0 to target), and \textbf{C}~(curriculum).                                                                                                     
  Table~\ref{tab:ablation} reveals three findings:                                                                                                         
  Curriculum (C) is the single strongest factor, yielding +4.20 avg.\ accuracy alone, while annealing (A) in isolation provides negligible gain (+0.22).
 T+A jumps to +4.52, a 20$\times$ amplification over A alone, demonstrating that entropy-adaptive zone sparsity 
  requires a matching schedule to be effective.
  The three-way combination T+A+C achieves the best overall result (64.00 avg., +5.2 AIME25), with each pairwise interaction contributing complementary    
  gains. T+A matches curriculum-only (63.70 vs.\ 63.38), and A+C (63.97) narrows the gap to the full system.
  This confirms that SEAD's token-level zone assignment, annealing, and curriculum address orthogonal bottlenecks: zone sparsity targets \emph{which}
  tokens receive gradient, annealing controls \emph{when} sparsity activates, and curriculum determines \emph{what order} examples are presented.

  \begin{table}[htbp]                                                                                                                                         
  \caption{$2^3$ ablation. T = Token SEAD, A = Annealing, C = Curriculum. The AIME$\Delta$ column highlights the supervision quality cascade: T alone hurts
   hard benchmarks; T+C recovers via reliable entropy signals. The full combination is super-additive.}                                                    
  \label{tab:ablation}                                                                                                                                   
  \centering                             
  \small
  \begin{tabular}{ccc|ccc}
  \toprule
  \textbf{T} & \textbf{A} & \textbf{C} & \textbf{Avg. Acc.} & \textbf{$\Delta$ vs. OPD} & \textbf{AIME25 $\Delta$} \\
  \midrule
             &            &            & 59.18 & +0.0 & +0.0 \\
  \checkmark &            &            & 59.50 & +0.32 & +2.0 \\
             & \checkmark &            & 59.40 & +0.22 & $-$1.7 \\
             &            & \checkmark & 63.38 & +4.20 & +4.6 \\
  \checkmark & \checkmark &            & 63.70 & +4.52 & +4.8 \\
  \checkmark &            & \checkmark & 63.42 & +4.24 & +3.4 \\
             & \checkmark & \checkmark & 63.97 & +4.79 & +4.5 \\
  \checkmark & \checkmark & \checkmark & 64.00 & +4.82 & +5.2 \\
  \bottomrule
  \end{tabular}
  \end{table}

\section{Analysis}
\label{sec:analysis}

\subsection{Token-Level Analysis}
\label{sec:token_analysis}


\begin{figure*}[h]
    \centering
    \begin{subfigure}[t]{0.48\textwidth}
        \centering
        \includegraphics[width=\textwidth]{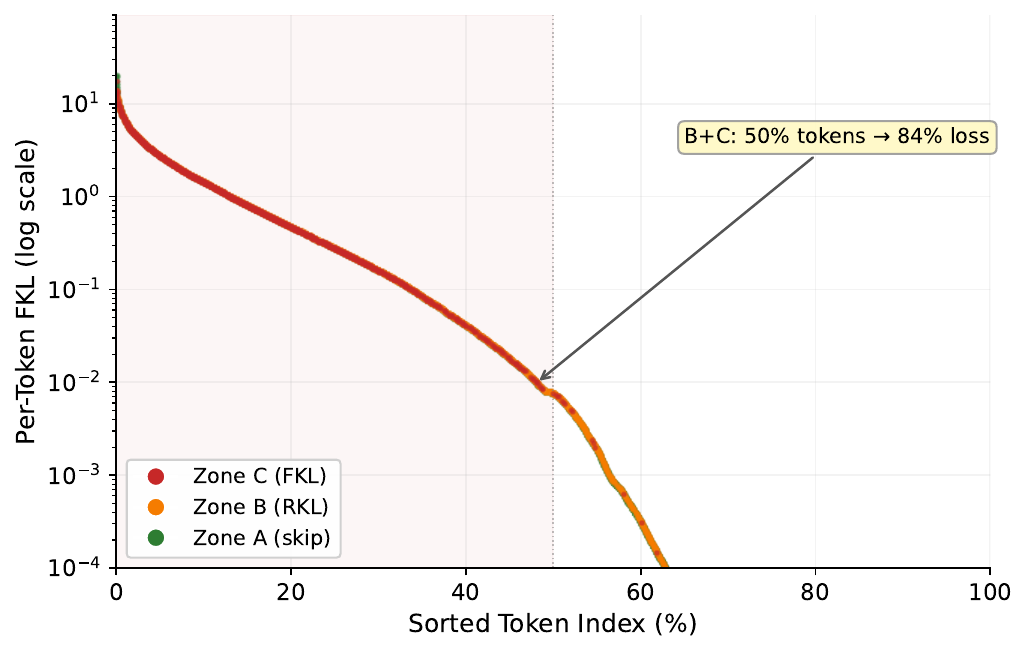}
        \caption{Sorted per-token FKL by zone. The top 50\% of tokens (Zones B+C) concentrate 90\% of total loss, while the remaining 50\% (Zone~A) carry only 10\%.}
        \label{fig:sead-panel-a}
    \end{subfigure}
    \hfill
    \begin{subfigure}[t]{0.48\textwidth}
        \centering
        \includegraphics[width=\textwidth]{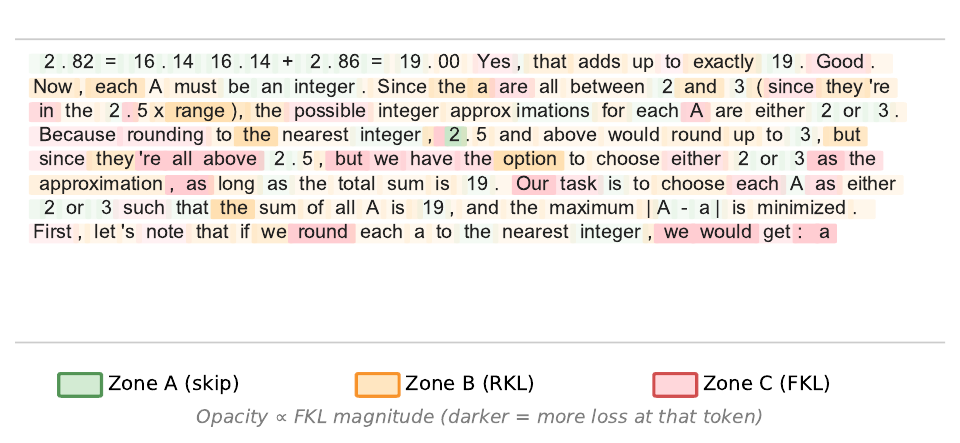}
        \caption{Token-level zone visualization on OLMo-3-7B output. Background color indicates zone assignment; opacity is proportional to FKL magnitude. Strategy forks (dark red) cluster at reasoning decision points.}
        \label{fig:sead-panel-b}
    \end{subfigure}

    \vspace{0.3cm}

    \begin{subfigure}[t]{0.48\textwidth}
        \centering
        \includegraphics[width=\textwidth]{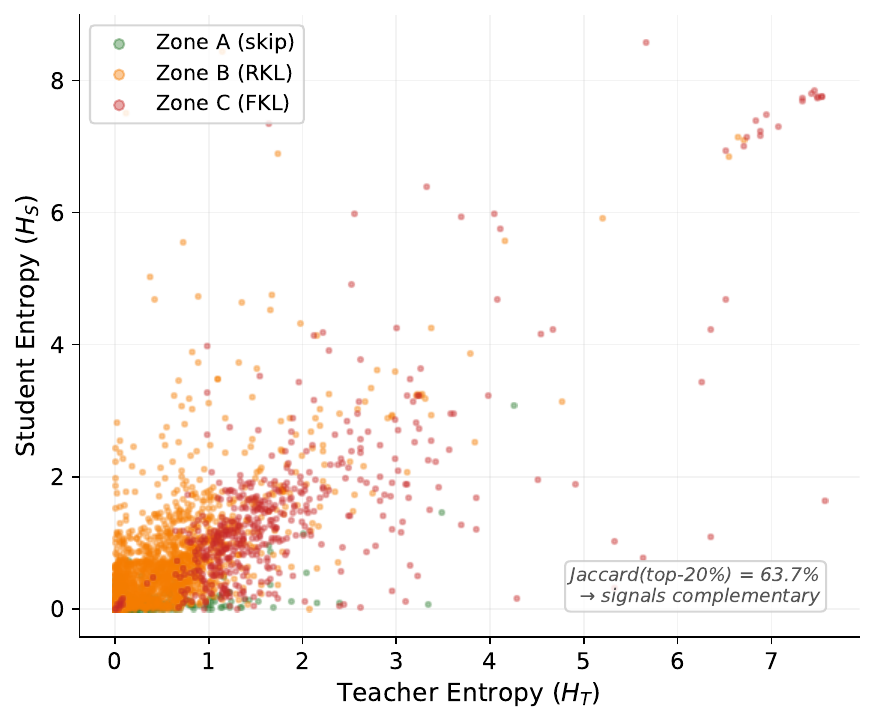}
        \caption{Zone assignment in joint entropy space ($H_T$ vs.\ $H_S$). Jaccard overlap of top-20\% sets is 63.7\%, confirming complementary signals from teacher and student entropy.}
        \label{fig:sead-panel-c}
    \end{subfigure}
    \hfill
    \begin{subfigure}[t]{0.48\textwidth}
        \centering
        \includegraphics[width=\textwidth]{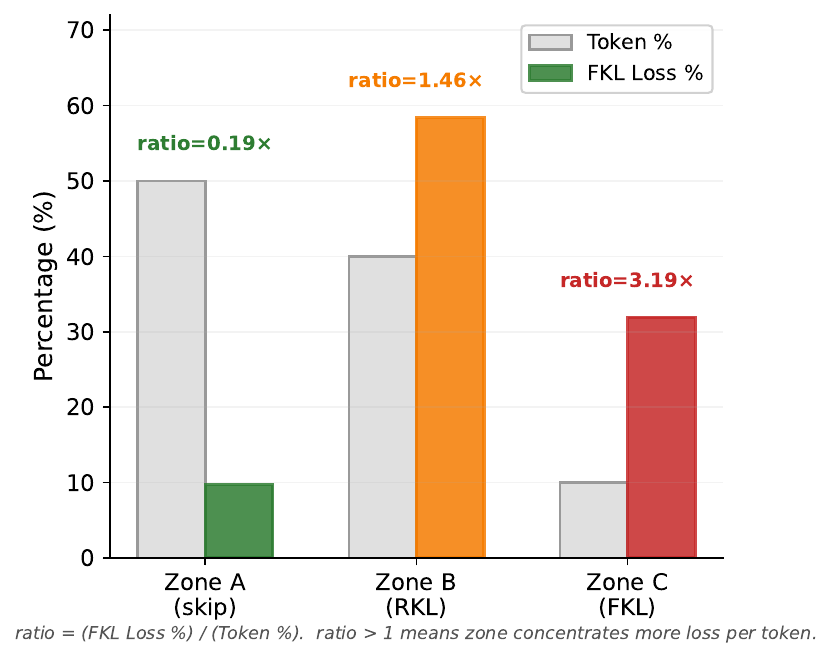}
        \caption{Loss concentration by zone. Zone~C (10\% of tokens) carries 31.9\% of FKL loss ($3.19\times$ concentration); Zone~B+C together achieve $1.81\times$ over uniform.}
        \label{fig:sead-panel-d}
    \end{subfigure}

    \caption{\textbf{SEAD zone analysis on OLMo-3-7B / OLMo-3.1-32B (50/40/10 config).} (a)~Sorted per-token FKL reveals a heavy-tailed loss distribution where 50\% of tokens dominate gradient signal. (b)~Zone coloring on actual model output shows Zone~C (red) tokens align with reasoning forks. (c)~Joint entropy space confirms teacher and student signals are partially complementary (Jaccard = 63.7\%). (d)~Loss concentration validates that Zone~A tokens are safe to skip ($0.19\times$ ratio) while Zone~C tokens concentrate $3.19\times$ more loss per token.}
    \label{fig:sead-analysis}
\end{figure*}

\paragraph{Entropy distribution and zone assignment.}
Zone~A tokens overwhelmingly correspond to deterministic computation steps and connectives, e.g., on OLMo-3 7B/32B (50/40/10 config), 73.2\% of computation tokens and 54.2\% of syntax tokens are assigned to Zone~A.
Zone~B tokens cluster at reasoning transitions where the student holds partial knowledge (17.2\% strategy forks, 52.7\% common tokens receiving moderate gradient), and RKL sharpens the student toward the teacher's peak mode.
Zone~C tokens appear at genuine reasoning forks---63.9\% are \emph{strategy fork} positions where both teacher and student entropy exceed the 75th percentile, representing critical decision points with multiple valid continuations where FKL's mode-covering property prevents catastrophic collapse.
Crucially, the 50/40/10 configuration routes \emph{zero} strategy fork tokens to Zone~A, ensuring every reasoning-critical position receives active supervision (Table~\ref{tab:zone-category-5040}).
Figure~\ref{fig:sead-panel-b} visualizes this on actual model output: Zone~C tokens (red) concentrate at branching points in the reasoning chain, while Zone~A tokens (green) cover predictable syntax and arithmetic with near-zero loss contribution.

\paragraph{Selection overlap and signal complementarity.}
With a Jaccard index of 63.7\% between the top-20\% teacher entropy and top-20\% student entropy positions, approximately 36\% of high-entropy tokens are unique to one signal (Figure~\ref{fig:sead-panel-c}).
Teacher entropy identifies positions where guidance is most valuable (multiple valid continuations), while student entropy identifies positions the student has already mastered.
This partial overlap justifies joint conditioning: neither signal alone captures the full picture.

\paragraph{Gradient concentration.}
Figure~\ref{fig:sead-panel-d} quantifies the concentration effect: Zone~A (50\% of tokens) carries only 9.7\% of total FKL loss, while Zone~B+C (50\% of tokens) concentrates 90.3\% of gradient signal.
The per-zone ratios are stark: Zone~C tokens carry $3.19\times$ their fair share of loss, Zone~B tokens carry $1.46\times$, and Zone~A tokens carry only $0.19\times$.
This means skipping Zone~A discards half the tokens while retaining over 90\% of the training signal, a highly efficient compute-quality tradeoff.
Figure~\ref{fig:sead-panel-a} shows this visually: the sorted FKL curve drops precipitously, with Zone~B and C tokens (orange, red) dominating the high-loss region while Zone~A tokens (green) cluster near zero.


\begin{table}[t]
\centering
\small
\setlength{\tabcolsep}{3.5pt}
\caption{Token categorization: P(zone $|$ category). The 50/40/10 config routes zero strategy forks to Zone~A, ensuring all reasoning-critical tokens receive supervision.}
\label{tab:zone-category-5040}
\begin{tabular}{@{}l c c c c@{}}
\toprule
Category &  Zone A & Zone B &  Zone C & Total \\
\midrule
Deterministic syntax & 54.2\% & 39.1\% & 6.7\% & 4{,}777 \\
Computation & 73.2\% & 21.8\% & 5.0\% & 2{,}564 \\
Reasoning transition & 10.8\% & 59.1\% & 30.0\% & 536 \\
Strategy fork & \textbf{0.0\%} & 52.7\% & \textbf{47.3\%} & 2{,}968 \\
Other & 57.4\% & 41.1\% & 1.6\% & 11{,}696 \\
\bottomrule
\end{tabular}
\end{table}
\subsection{Training Dynamics}
\label{sec:dynamics}

A known failure mode of vanilla OPD is premature entropy collapse, where the student model overly concentrates probability mass on the teacher's peak mode before exploring alternative reasoning paths. SEAD effectively mitigates this collapse. As illustrated in Appendix~\ref{app:train_dynamics}, vanilla OPD exhibits a clear decreasing trend in average token-level entropy during training, whereas SEAD maintains stable, consistently higher entropy levels. This preserves necessary distributional diversity while still sharpening the student's outputs. 

\vspace{-0.9em}
\section{Related Work}
\label{sec:related}

\paragraph{On-Policy Distillation.}
GKD \citep{agarwal2024gkd} established on-policy distillation for LMs. Extensions include RKL for reasoning \citep{lu2025thinking_machines}, self-distillation (OPSD, \citealp{opsd2026}), and theoretical unification with dense KL-constrained RL (G-OPD, \citealp{gopd2026}). Recent systems work focuses on scaling \citep{reopold2026, lightning_opd2026}.
\vspace{-0.9em}

\paragraph{Token-Level Selection.}
EOPD \citep{eopd2026} switches FKL/RKL based on teacher entropy alone. TIP \citep{tip2026} uses student entropy $\times$ divergence for weighting with a single divergence. SE-KD \citep{sekd2026} and SCOPE \citep{scope2026} explore other selection criteria. SEAD differs via (i)~\emph{joint} teacher-student partitioning, and (ii)~per-zone divergence switching.
\vspace{-0.9em}

\paragraph{Divergence Scheduling.}
Reopold \citep{reopold2026} uses a discrete two-stage schedule; PACED \citep{paced2026} finds forward-then-reverse optimal; DRKL \citep{drkl2026} addresses entropy collapse. Our annealing generalizes these to continuous transitions motivated by evolving zone composition.
\vspace{-0.5em}
\paragraph{Curriculum for Reasoning.}
Curriculum has been studied for RL-based reasoning~\citep{parashar2025curriculum, wen2025light}. However, \textbf{No prior work applies curriculum to OPD}, despite it being flagged as an open problem \citep{opsd2026, rethinking_opd2026}. Our contribution is showing \emph{why} it is uniquely necessary: the modulated PL condition (Assumption~\ref{as:quality}) formalizes supervision quality degradation specific to OPD.

\section{Conclusion and Limitations}
\label{sec:conclusion}
\vspace{-0.9em}
In conclusion, SEAD addresses the dependence of teacher supervision on student competence in on-policy distillation by utilizing entropy as a unified signal. It establishes a mutually reinforcing framework operating across three levels: applying adaptive divergences to informative tokens, temporally annealing from exploration to refinement, and employing a competence-gated prompt curriculum. However, this work has several limitations. First, the curriculum relies on static difficulty scores rather than computationally expensive adaptive re-estimation. Second, our method may apply to other domains (e.g., code) for future work. Finally, SEAD's efficacy on very long reasoning chains requires further investigation.



\bibliographystyle{plainnat}
\bibliography{references}

@inproceedings{agarwal2024gkd,
  title={On-Policy Distillation of Language Models: Learning from Self-Generated Mistakes},
  author={Agarwal, Rishabh and Vieillard, Nino and Zhou, Yongchao and Stanczyk, Piotr and Ramos, Sabela and Geist, Matthieu and Bachem, Olivier},
  booktitle={International Conference on Learning Representations},
  year={2024},
  url={https://openreview.net/forum?id=3zKtaqxLhW}
}

@article{lu2025thinking_machines,
  title={On-Policy Distillation},
  author={Lu, Kevin and {Thinking Machines Lab}},
  journal={Thinking Machines Lab: Connectionism},
  year={2025},
  doi={10.64434/tml.20251026},
  url={https://thinkingmachines.ai/blog/on-policy-distillation}
}

@article{reopold2026,
  title={Scaling Reasoning Efficiently via Relaxed On-Policy Distillation},
  author={Ko, Jongwoo and Abdali, Sara and Kim, Young Jin and Chen, Tianyi and Cameron, Pashmina},
  journal={arXiv preprint arXiv:2603.11137},
  year={2026}
}

@article{eopd2026,
  title={Entropy-Aware On-Policy Distillation of Language Models},
  author={Jin, Woogyeol and Min, Taywon and Yang, Yongjin and Kadhe, Swanand Ravindra and Zhou, Yi and Wei, Dennis and Baracaldo, Nathalie and Lee, Kimin},
  journal={arXiv preprint arXiv:2603.07079},
  year={2026}
}

@article{gopd2026,
  title={Learning beyond Teacher: Generalized On-Policy Distillation with Reward Extrapolation},
  author={Yang, Wenkai and Liu, Weijie and Xie, Ruobing and Yang, Kai and Yang, Saiyong and Lin, Yankai},
  journal={arXiv preprint arXiv:2602.12125},
  year={2026}
}

@article{opsd2026,
  title={Self-Distilled Reasoner: On-Policy Self-Distillation for Large Language Models},
  author={Zhao, Siyan and Xie, Zhihui and Liu, Mengchen and Huang, Jing and Pang, Guan and Chen, Feiyu and Grover, Aditya},
  journal={arXiv preprint arXiv:2601.18734},
  year={2026}
}

@article{opd_survey_2026,
  title={A Survey of On-Policy Distillation for Large Language Models},
  author={Song, Mingyang and Zheng, Mao},
  journal={arXiv preprint arXiv:2604.00626},
  year={2026}
}

@article{lightning_opd2026,
  title={Lightning {OPD}: Efficient Post-Training for Large Reasoning Models with Offline On-Policy Distillation},
  author={Wu, Yecheng and others},
  journal={arXiv preprint arXiv:2604.13010},
  year={2026}
}

@article{tip2026,
  title={{TIP}: Token Importance in On-Policy Distillation},
  author={Xu, Yuanda and Sang, Hejian and Zhou, Zhengze and He, Ran and Wang, Zhipeng and Geramifard, Alborz},
  journal={arXiv preprint arXiv:2604.14084},
  year={2026}
}

@article{sekd2026,
  title={Rethinking Selective Knowledge Distillation},
  author={Tavor, Almog and Ebenspanger, Itay and Cnaan, Neil and Geva, Mor},
  journal={arXiv preprint arXiv:2602.01395},
  year={2026}
}

@article{scope2026,
  title={{SCOPE}: Signal-Calibrated On-Policy Distillation Enhancement with Dual-Path Adaptive Weighting},
  author={Zheng, Binbin and Ma, Xing and Liang, Yiheng and Ruan, Jingqing and Fu, Xiaoliang and Lin, Kepeng and Zhu, Benchang and Zeng, Ke and Cai, Xunliang},
  journal={arXiv preprint arXiv:2604.10688},
  year={2026}
}

@article{paced2026,
  title={{PACED}: Distillation at the Frontier of Student Competence},
  author={Xu, Yuanda and Sang, Hejian and Zhou, Zhengze and He, Ran and Wang, Zhipeng},
  journal={arXiv preprint arXiv:2603.11178},
  year={2026}
}

@article{drkl2026,
  title={Diversity-Aware Reverse {Kullback-Leibler} Divergence for Large Language Model Distillation},
  author={Luong, Hoang-Chau and others},
  journal={arXiv preprint arXiv:2604.00223},
  year={2026}
}

@article{rethinking_opd2026,
  title={Rethinking On-Policy Distillation of Large Language Models: Phenomenology, Mechanism, and Recipe},
  author={Li, Yaxuan and Zuo, Yuxin and He, Bingxiang and Zhang, Jinqian and Xiao, Chaojun and Qian, Cheng and Yu, Tianyu and Gao, Huan-ang and Yang, Wenkai and Liu, Zhiyuan and Ding, Ning},
  journal={arXiv preprint arXiv:2604.13016},
  year={2026}
}

@inproceedings{karimi2016linear,
  title={Linear convergence of gradient and proximal-gradient methods under the polyak-{\l}ojasiewicz condition},
  author={Karimi, Hamed and Nutini, Julie and Schmidt, Mark},
  booktitle={Joint European conference on machine learning and knowledge discovery in databases},
  pages={795--811},
  year={2016},
  organization={Springer}
}

@inproceedings{platanios2019competence,
  title={Competence-based curriculum learning for neural machine translation},
  author={Platanios, Emmanouil Antonios and Stretcu, Otilia and Neubig, Graham and Poczos, Barnabas and Mitchell, Tom},
  booktitle={Proceedings of the 2019 conference of the North American chapter of the association for computational linguistics: human language technologies, volume 1 (long and short papers)},
  pages={1162--1172},
  year={2019}
}

@article{parashar2025curriculum,
  title={Curriculum reinforcement learning from easy to hard tasks improves LLM reasoning},
  author={Parashar, Shubham and Gui, Shurui and Li, Xiner and Ling, Hongyi and Vemuri, Sushil and Olson, Blake and Li, Eric and Zhang, Yu and Caverlee, James and Kalathil, Dileep and others},
  journal={arXiv preprint arXiv:2506.06632},
  year={2025}
}

@inproceedings{wen2025light,
  title={Light-r1: Curriculum sft, dpo and rl for long cot from scratch and beyond},
  author={Wen, Liang and Cai, Yunke and Xiao, Fenrui and He, Xin and An, Qi and Duan, Zhenyu and Du, Yimin and Liu, Junchen and Tanglifu, Tanglifu and Lv, Xiaowei and others},
  booktitle={Proceedings of the 63rd Annual Meeting of the Association for Computational Linguistics (Volume 6: Industry Track)},
  pages={318--327},
  year={2025}
}

@article{dapo2025,
  title={{DAPO}: An Open-Source {LLM} Reinforcement Learning System at Scale},
  author={Yu, Qiying and Zhang, Zheng and Zhu, Ruofei and Yuan, Yufeng and Zuo, Xiaochen and Yue, Yu and Dai, Weinan and Fan, Tiantian and Liu, Gaohong and Liu, Lingjun and others},
  journal={arXiv preprint arXiv:2503.14476},
  year={2025}
}

@article{shao2024deepseekmath,
  title={Deepseekmath: Pushing the limits of mathematical reasoning in open language models},
  author={Shao, Zhihong and Wang, Peiyi and Zhu, Qihao and Xu, Runxin and Song, Junxiao and Bi, Xiao and Zhang, Haowei and Zhang, Mingchuan and Li, YK and Wu, Yang and others},
  journal={arXiv preprint arXiv:2402.03300},
  year={2024}
}

\appendix
\section{Theoretical Proofs}
\label{app:proofs}

\subsection{Full Assumptions for Theorem~\ref{thm:sparse-gradient}}

\begin{assumption}[Bounded score functions]\label{ass:bounded-score}
There exists a constant $G > 0$ such that for every token position $t$ and vocabulary element $v \in \mathcal{V}$, $\|\nabla_{\theta}\log\pi_{\theta}(v \mid \mathbf{c}_t)\| \leq G$.
\end{assumption}

\begin{assumption}[Low-entropy agreement in Zone~A]\label{ass:agreement}
For every token position $t \in \mathcal{A}$, the teacher and student distributions satisfy $H_{\mathrm{te}}(t) \leq \tau$ and $H_{\theta}(t) \leq \tau$, and $\mathrm{TV}(\pi_{\mathrm{te}}(\cdot \mid \mathbf{c}_t), \pi_{\theta}(\cdot \mid \mathbf{c}_t)) \leq \delta(\tau)$, where $\delta(\tau) \to 0$ as $\tau \to 0$. A sufficient condition is that both distributions place mass $\geq 1 - \sqrt{\tau/2}$ on a common mode $v^{\star}_t$.
\end{assumption}

\begin{assumption}[Bounded log-density ratio]\label{ass:bounded-ratio}
For all $v \in \mathcal{V}$ and token positions $t$, $|\log\pi_\theta(v \mid \mathbf{c}_t) - \log\pi_{\mathrm{te}}(v \mid \mathbf{c}_t)| \leq B_{\max}$. In Zone~A, the low-entropy agreement (Assumption~\ref{ass:agreement}) guarantees that for the dominant mode $v^\star_t$: $|\log\pi_\theta(v^\star_t) - \log\pi_{\mathrm{te}}(v^\star_t)| \leq \delta(\tau)/(1-\sqrt{\tau/2})$, while off-mode tokens ($v \neq v^\star_t$) satisfy $|\log\pi_\theta(v) - \log\pi_{\mathrm{te}}(v)| \leq \log|\mathcal{V}|$ trivially since all probabilities are at least $1/|\mathcal{V}|$ under temperature-bounded sampling.
\end{assumption}

\subsection{Proof of Theorem~\ref{thm:sparse-gradient}}
\label{app:proof-sparse}

\begin{proof}
Write the full gradient as a convex decomposition over zones:
\[
  g = \frac{|\mathcal{A}|}{N}\,\bar{g}_{\mathcal{A}} + \frac{|\mathcal{B}\cup\mathcal{C}|}{N}\,\hat{g},
\]
where $\bar{g}_{\mathcal{A}} = \frac{1}{|\mathcal{A}|}\sum_{t\in\mathcal{A}}\nabla_{\theta}\mathcal{L}(t)$. Since $|\mathcal{B}\cup\mathcal{C}|/N = 1-s$:
\[
  g - (1-s)\,\hat{g} = s\,\bar{g}_{\mathcal{A}}.
\]
It suffices to bound $\|\nabla_\theta\mathcal{L}(t)\|$ for each $t\in\mathcal{A}$.

\paragraph{Forward KL gradient at a Zone~A token.}
For forward KL, $\nabla_{\theta}\,L_{\mathrm{FKL}}(t) = -\sum_{v\in\mathcal{V}} \pi_{\mathrm{te}}(v\mid \mathbf{c}_t)\,\nabla_{\theta}\log\pi_{\theta}(v\mid \mathbf{c}_t)$.
Using the score-function identity $\sum_v \pi_{\theta}(v\mid \mathbf{c}_t)\,\nabla_{\theta}\log\pi_{\theta}(v\mid \mathbf{c}_t) = \mathbf{0}$:
\begin{align}
  \nabla_{\theta}\,L_{\mathrm{FKL}}(t)
  &= -\sum_{v} \bigl[\pi_{\mathrm{te}}(v\mid \mathbf{c}_t) - \pi_{\theta}(v\mid \mathbf{c}_t)\bigr]\,\nabla_{\theta}\log\pi_{\theta}(v\mid \mathbf{c}_t). \label{eq:fkl-expanded}
\end{align}
By Assumption~\ref{ass:bounded-score}, $\|\nabla_{\theta}\,L_{\mathrm{FKL}}(t)\| \leq 2\,G\,\mathrm{TV}(\pi_{\mathrm{te}},\pi_{\theta}) \leq 2\,G\,\delta(\tau)$.

When $H_{\mathrm{te}}(t) \leq \tau$, the teacher places mass $\geq 1-\sqrt{\tau/2}$ on a single token $v^{\star}$ (via the binary entropy inequality $h(p) = -p\log p - (1{-}p)\log(1{-}p) \geq 2p^2$ for small $p$, applied to the off-mode mass). Similarly for the student. Then from~\eqref{eq:fkl-expanded}:
$\|\nabla_{\theta}\,L_{\mathrm{FKL}}(t)\| \leq G(\sqrt{2\tau} + 2\,\delta(\tau))$.

\paragraph{Reverse KL gradient at a Zone~A token.}
The gradient of $L_{\mathrm{RKL}}(t) = \mathrm{KL}(\pi_\theta(\cdot\mid\mathbf{c}_t) \| \pi_{\mathrm{te}}(\cdot\mid\mathbf{c}_t))$ is:
\begin{align}
  \nabla_\theta L_{\mathrm{RKL}}(t)
  &= \sum_{v\in\mathcal{V}} \nabla_\theta\!\left[\pi_\theta(v\mid\mathbf{c}_t)\bigl(\log\pi_\theta(v\mid\mathbf{c}_t) - \log\pi_{\mathrm{te}}(v\mid\mathbf{c}_t)\bigr)\right] \nonumber\\
  &= \sum_{v\in\mathcal{V}} \pi_\theta(v\mid\mathbf{c}_t)\bigl[\log\pi_\theta(v\mid\mathbf{c}_t) - \log\pi_{\mathrm{te}}(v\mid\mathbf{c}_t) + 1\bigr]\,\nabla_\theta\log\pi_\theta(v\mid\mathbf{c}_t) \nonumber\\
  &= \sum_{v\in\mathcal{V}} \pi_\theta(v\mid\mathbf{c}_t)\bigl[\log\pi_\theta(v\mid\mathbf{c}_t) - \log\pi_{\mathrm{te}}(v\mid\mathbf{c}_t)\bigr]\,\nabla_\theta\log\pi_\theta(v\mid\mathbf{c}_t),
  \label{eq:rkl-grad-closed}
\end{align}
where the last equality uses the score function identity $\sum_v \pi_\theta(v)\nabla_\theta\log\pi_\theta(v) = \mathbf{0}$.

In Zone~A, both distributions concentrate mass $\geq 1-\sqrt{\tau/2}$ on a common mode $v^\star$ (Assumption~\ref{ass:agreement}). Splitting the sum into on-mode ($v = v^\star$) and off-mode ($v \neq v^\star$) contributions:
\begin{itemize}[leftmargin=1.5em]
  \item \textbf{On-mode:} $\pi_\theta(v^\star) \geq 1-\sqrt{\tau/2}$ and $|\log\pi_\theta(v^\star) - \log\pi_{\mathrm{te}}(v^\star)| \leq \delta(\tau)/(1-\sqrt{\tau/2})$.
  \item \textbf{Off-mode:} Total mass $\leq \sqrt{\tau/2}$, and each log-ratio is bounded by $\log|\mathcal{V}|$ (Assumption~\ref{ass:bounded-ratio}).
\end{itemize}
By Assumption~\ref{ass:bounded-score}:
\begin{align}
  \|\nabla_\theta L_{\mathrm{RKL}}(t)\|
  &\leq G\left[(1-\sqrt{\tau/2})\cdot\frac{\delta(\tau)}{1-\sqrt{\tau/2}} + \sqrt{\tau/2}\cdot\log|\mathcal{V}|\right] \nonumber\\
  &= G\bigl(\delta(\tau) + \sqrt{\tau/2}\,\log|\mathcal{V}|\bigr) = \mathcal{O}\!\bigl(G\sqrt{\tau}\,\log|\mathcal{V}|\bigr).
  \label{eq:rkl-zone-a-bound}
\end{align}

\paragraph{Aggregation.}
The FKL bound gives $G(\sqrt{2\tau} + 2\,\delta(\tau))$ and the RKL bound gives $G(\delta(\tau) + \sqrt{\tau/2}\,\log|\mathcal{V}|)$. Taking the maximum:
$\|\bar{g}_{\mathcal{A}}\| \leq G(\sqrt{2\tau} + 2\,\delta(\tau) + \sqrt{\tau/2}\,\log|\mathcal{V}|)$.
Therefore $\|g - (1-s)\hat{g}\| = s\|\bar{g}_{\mathcal{A}}\| \leq s\,G(\sqrt{2\tau} + 2\,\delta(\tau) + \sqrt{\tau/2}\,\log|\mathcal{V}|) = O(s\,G\sqrt{\tau}\,\log|\mathcal{V}|)$.
In practice $\log|\mathcal{V}| = O(\log 10^5) \approx 11.5$ is a moderate constant, so the bound simplifies to $O(s\,G\sqrt{\tau})$ up to logarithmic factors.
\end{proof}

\subsection{Full Assumptions and Proof of Theorem~\ref{thm:curriculum}}
\label{app:proof-curriculum}

The analysis requires four assumptions. Assumptions~\ref{as:smooth} and~\ref{as:var} are standard; the novelty is Assumption~\ref{as:quality}.

\begin{assumption}[Smoothness]\label{as:smooth}
Each $\ell_i$ is $\beta$-smooth: $\|\nabla \ell_i(\theta) - \nabla \ell_i(\theta')\| \leq \beta\|\theta - \theta'\|$.
\end{assumption}

\begin{assumption}[Bounded stochastic gradients]\label{as:var}
$\mathbb{E}[g_i(\theta)] = \nabla \ell_i(\theta)$ and $\mathbb{E}\|g_i(\theta) - \nabla \ell_i(\theta)\|^2 \leq \sigma^2$.
\end{assumption}

\begin{assumption}[Stage mastery]\label{as:master}
For each stage $s$, after $T_s$ SGD steps on $\mathcal{Q}_{\leq s}$, the student achieves $p_i(\theta) \geq \alpha_{\min}$ for all $q_i \in \mathcal{Q}_{\leq s}$, where $\phi(\alpha_{\min}) \geq c > 0$.
\end{assumption}

\begin{proof}[Proof of Theorem~\ref{thm:curriculum}]

\textbf{Step 1: Per-step descent.}
By $\beta$-smoothness and $\eta \leq 1/\beta$:
\begin{equation}
  \mathbb{E}_g[\ell_i(\theta_t) - \ell_i(\theta_{t+1})] \geq \frac{\eta}{2}\|\nabla \ell_i(\theta_t)\|^2 - \frac{\eta^2\beta\sigma^2}{2}.
  \label{eq:app-descent}
\end{equation}

\textbf{Step 1b: Lifting to stage-level loss.}
At each step, the SGD algorithm samples a prompt $i$ uniformly from the eligible set $\mathcal{D}_s$ at stage $s$. Define the stage-level average loss $L_s(\theta) = \frac{1}{|\mathcal{D}_s|}\sum_{i \in \mathcal{D}_s} \ell_i(\theta)$. Taking expectation over $i \sim \mathrm{Uniform}(\mathcal{D}_s)$ in~\eqref{eq:app-descent}:
\begin{equation}
  \mathbb{E}_{i,g}[L_s(\theta_t) - L_s(\theta_{t+1})] \geq \frac{\eta}{2}\,\mathbb{E}_i\!\left[\|\nabla \ell_i(\theta_t)\|^2\right] - \frac{\eta^2\beta\sigma^2}{2}.
  \label{eq:app-descent-global}
\end{equation}
Crucially, $\mathbb{E}_i[\|\nabla \ell_i\|^2] \geq \|\mathbb{E}_i[\nabla \ell_i]\|^2 = \|\nabla L_s(\theta_t)\|^2$ by Jensen, so the bound on the \emph{average} of squared norms is at least as strong as what a global PL condition on $L_s$ would require. We proceed using $\mathbb{E}_i[\|\nabla \ell_i\|^2]$ directly.

\textbf{Step 2: Modulated PL.}
By Assumption~\ref{as:quality}: $\|\nabla \ell_i(\theta_t)\|^2 \geq \phi(p_i(\theta_t))\,\mu\,(\ell_i(\theta_t) - \ell_i^\star)$.
Averaging over $i \sim \mathrm{Uniform}(\mathcal{D}_s)$ and substituting into~\eqref{eq:app-descent-global}:
\begin{equation}
  \mathbb{E}_{i,g}[L_s(\theta_t) - L_s(\theta_{t+1})] \geq \frac{\eta\mu}{2}\,\mathbb{E}_i\!\left[\phi(p_i(\theta_t))\,(\ell_i(\theta_t) - \ell_i^\star)\right] - \frac{\eta^2\beta\sigma^2}{2}.
  \label{eq:app-descent-quality}
\end{equation}

\textbf{Step 3: Curriculum analysis.}
Under curriculum, $\phi(p_i(\theta)) \geq c$ on all eligible prompts (Assumption~\ref{as:master}). From Step 2, using $\phi \geq c$ uniformly:
$\mathbb{E}_i[\phi(p_i)(\ell_i - \ell_i^\star)] \geq c\,(L_s(\theta_t) - L_s^\star)$.
This yields a standard PL descent with effective constant $\mu_{\mathrm{eff}} = c\mu$. Applying the SGD convergence rate under PL~\citep{karimi2016linear} with optimal step size $\eta = c\mu\epsilon/(\beta\sigma^2)$:
\[
  T^{\textup{C}} = O\!\left(\frac{\beta\sigma^2}{(c\,\mu)^2\,\epsilon} + \frac{1}{c\,\mu}\log\frac{1}{\epsilon}\right).
\]
The first term dominates in the stochastic regime ($\sigma^2 > 0$).

\textbf{Step 4: Uniform analysis.}
Under uniform sampling over all $n$ prompts, define $\bar\phi^{\,\textup{U}} = \frac{1}{n}\sum_{i=1}^n \phi(p_i(\theta))$. The expected descent from Step 2 becomes:
\[
  \mathbb{E}_i[\phi(p_i)\,(\ell_i - \ell_i^\star)].
\]
Since competence $\phi(p_i)$ and suboptimality $(\ell_i - \ell_i^\star)$ are negatively correlated (easy prompts have high $\phi$ but low suboptimality), by Chebyshev's sum inequality:
\[
  \mathbb{E}_i[\phi(p_i)\,(\ell_i - \ell_i^\star)] \leq \bar\phi^{\,\textup{U}} \cdot (L(\theta) - L^\star).
\]
Therefore the effective PL constant under uniform sampling is \emph{at most} $\bar\phi^{\,\textup{U}}\mu$ (the true constant is even smaller due to the negative correlation, making uniform sampling \emph{strictly worse} than this upper bound suggests). Applying the same SGD convergence machinery:
\[
  T^{\textup{U}} \geq \Omega\!\left(\frac{\beta\sigma^2}{(\bar\phi^{\,\textup{U}}\,\mu)^2\,\epsilon} + \frac{1}{\bar\phi^{\,\textup{U}}\,\mu}\log\frac{1}{\epsilon}\right).
\]
Because hard prompts contribute $\phi \approx 0$ early in training, $\bar\phi^{\,\textup{U}} < c$.

\textbf{Step 5: Speed-up.}
In the stochastic regime, the ratio of convergence times is:
$T^{\textup{U}}/T^{\textup{C}} \geq \Omega\!\left((c/\bar\phi^{\textup{U}})^2\right)$.
Worst case: the student is initially competent only on stage~1 ($1/S$ fraction of prompts), so $\bar\phi^{\textup{U}} \leq c/S$, giving speedup $\geq \Omega(S^2)$. Even in the log-dominated (deterministic) regime, $T^{\textup{U}}/T^{\textup{C}} = \Theta(c/\bar\phi^{\textup{U}}) = \Theta(S)$.

\textbf{Step 6: OPD specificity.}
~
\begin{itemize}[leftmargin=1.5em]
  \item \textbf{Off-policy KD:} $\phi \equiv 1$ (teacher traces always coherent). No curriculum advantage.
  \item \textbf{RL:} Binary reward gives unbiased (if high-variance) gradients regardless of competence. OPD is different: teacher logits on incoherent rollouts are \emph{biased}, not just noisy. The $\phi$-modulation captures gradient \emph{norm} degradation, not just variance.
\end{itemize}
\end{proof}

\section{Additional Experiment Setup}
\subsection{Licenses for Existing Assets}
\label{app:asset-licenses}

We use only existing models, datasets, benchmarks, and software frameworks whose terms permit research use. Table~\ref{tab:asset_licenses} summarizes the main assets used in this work. We use these assets for training, evaluation, and comparison only, and we do not redistribute third-party model weights, benchmark data, or dataset contents unless their licenses explicitly permit redistribution. When releasing code, we will include links to the original asset sources, preserve copyright and license notices, and specify the exact versions used for all reproducibility-critical dependencies.

\begin{table}[h]
\centering
\small
\resizebox{\textwidth}{!}{
\begin{tabular}{lll}
\toprule
\textbf{Asset} & \textbf{Use in this paper} & \textbf{License / terms} \\
\midrule
OLMo model family & Student/teacher models & Apache-2.0, per official release terms \\
NVIDIA Nemotron model family & Student/teacher models & NVIDIA Open Model License / applicable Llama terms \\
DAPO-Math-17K & Training data & Apache-2.0, per dataset release \\
MATH-500 & Evaluation benchmark & Original benchmark/provider terms \\
Minerva-Math & Evaluation benchmark & Original benchmark/provider terms \\
OlympiadBench & Evaluation benchmark & MIT License, per official repository \\
AMC/AIME benchmarks & Evaluation benchmarks & Original competition/provider terms \\
\texttt{verl} & Training framework & Apache-2.0 \\
\bottomrule
\end{tabular}
}
\caption{Existing assets used in the paper. For assets with provider-specific or competition-specific terms, we use them only for evaluation and do not redistribute their contents.}
\label{tab:asset_licenses}
\end{table}

All third-party assets are used consistently with their intended research or evaluation purposes. For benchmarks derived from mathematical competitions or external providers, we report aggregate evaluation results only and do not claim ownership over the underlying problems. The proposed method itself does not require any proprietary dataset or private user data.

\subsection{Training Details.}
\label{app:train_detail}
During each iteration, we sample $B = 96$ prompts and generate one rollout per    
  prompt using a
  sampling temperature of 0.7 and a maximum generation length of 16{,}384 tokens.   
  Following teacher                                                                 
  scoring, we perform $K = 1$ gradient update with a clipping parameter $\epsilon =
  0.2$. We use a peak                    
  learning rate of $5 \times 10^{-9}$ with a linear warmup and decay schedule over a
   single training epoch.
  For the SEAD configuration, the regularization weights are set to $\rho_A = 50$,
  $\rho_B = 40$, and $\rho_C = 10$.
  The annealing factor follows a cosine schedule, decaying smoothly from
  $\alpha_{\text{start}} = 0.8$ to $\alpha_{\text{end}} = 0.0$.
  For curriculum learning, we adopt a smooth competence-based ordering inspired by
  \citet{platanios2019competence}.
  We first estimate per-problem difficulty as $d_i = 1 - \hat{a}_i$, where
  $\hat{a}_i$ is the pass rate computed
  from $K = 8$ independent rollouts of the student model prior to training. Examples
   are then sorted by
  difficulty in ascending order and presented easy-to-hard in a single pass without
  replacement,
  with small Gaussian noise ($\sigma = 0.02$) added to difficulty scores for
  stochastic tie-breaking.
  This provides a smooth, continuous curriculum that naturally increases problem
  difficulty as training
  progresses, without requiring discrete phase boundaries or explicit competence
  thresholds.

\section{Additional Experimental Results}

\subsection{Results on Nemotron}
We present the results on Nemotron model pair in Table~\ref{tab:main_nemotron}.
\label{app:exp_nemotron}
\begin{table}[t]
\caption{Results on Nemotron model pair (Nemotron-8B student, Nemotron-49B teacher).}
\label{tab:main_nemotron}
\centering
\small
\resizebox{\textwidth}{!}{
\begin{tabular}{lccccccc}
\toprule
\textbf{Method} & \textbf{MATH-500} & \textbf{Minerva} & \textbf{Olympiad} & \textbf{AMC23} & \textbf{AIME24} & \textbf{AIME25} & \textbf{Avg.} \\
\midrule
\multicolumn{8}{l}{\textit{Off-the-shelf Models}} \\
\midrule
Nemotron-8B (Student)    & 91.2 & 52.9 & 61.0 & 93.3 & 61.0 & 50.8 & 68.3 \\
Nemotron-49B (Teacher)   & 94.2 & 57.4 & 63.9 & 94.2 & 60.4 & 52.9 & 70.5 \\
\midrule
\multicolumn{8}{l}{\textit{Baselines}} \\
\midrule
OPSD                      & 90.8 & 49.6 & 61.9 & 93.4 & 63.2 & 48.9 & 67.9 \\
GRPO                     & 92.6 & 53.7 & 60.9 & 94.6 & 64.0 & 50.2 & 69.3 \\
OPD (RKL)                & 89.8 & 53.3 & 63.3 & 93.9 & 63.0 & 49.8 & 68.8 \\
\midrule
\multicolumn{8}{l}{\textbf{\textit{Proposed (SEAD components)}}} \\
\midrule
KL Annealing only                         & 93.4 & 56.6 & 62.5 & 94.3 & 62.7 & 50.4 & 69.9 \\
Token Zones only                & 92.0 & 54.0 & 62.1 & 93.8 & 63.1 & 51.5 & 69.4 \\
Token Zones + KL Annealing  & 90.4 & 54.8 & 61.6   & 94.7 & 64.6 & 51.0 & 69.5   \\
\bottomrule
\end{tabular}
}
\end{table}

\subsection{Training Dynamics}
\label{app:train_dynamics}

A known failure mode of OPD is premature entropy collapse: if the student is trained with reverse KL from the beginning, it is encouraged to concentrate probability mass on the teacher's current peak mode before it has explored alternative valid reasoning paths. SEAD addresses this through the interaction of token selection and KL annealing. Early in training, the coefficient $\alpha$ is high, so Zone C tokens receive forward-KL supervision; later, $\alpha$ decays toward zero, shifting the active objective toward reverse KL refinement. 

The key distinction from a global annealing baseline is that SEAD does not apply this uniformly across all tokens. Roughly 50\% of tokens are assigned to Zone A and skipped, because both teacher and student are already confident on them. Consequently, annealing only affects the subset of tokens where divergence choice matters: Zone C reasoning forks, where forward KL preserves the teacher's support, and Zone B partially mastered tokens, where reverse KL sharpens the student toward the teacher's preferred continuation. 

Figure~\ref{fig:entropy_dynamics} shows the average entropy for different methods during the training. Among these methods, vanilla OPD shows a clear decreasing trend in entropy, where all our SEAD-related methods keep the entropy at a similar level and higher than vanilla OPD, which proves the ability of our method to prevent entropy collapse.

\begin{figure}[h]
\centering
\includegraphics[width=\textwidth]{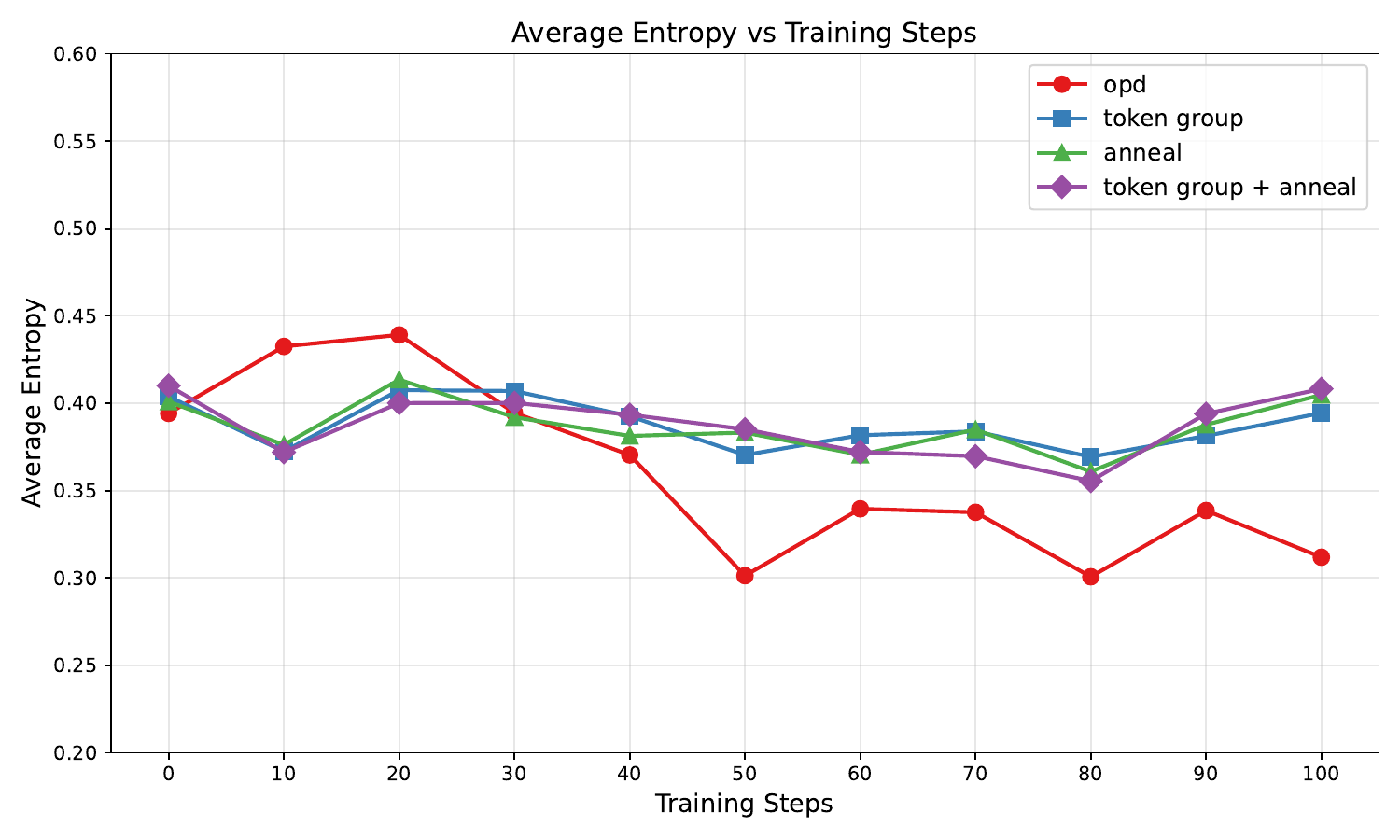}
\caption{Average token-level entropy over training steps. Vanilla OPD (red) suffers progressive entropy collapse ($0.39 \to 0.31$), while SEAD's token selection (blue), annealing (green), and their combination (purple) maintain stable entropy throughout training, preserving the diversity for token distribution.}
\label{fig:entropy_dynamics}
\end{figure}

\section{Broader Impacts}
\label{app:broader-impacts}

This work aims to improve the efficiency of distilling reasoning capabilities from larger language models into smaller student models. By reducing the amount of token-level supervision required during on-policy distillation, SEAD may lower the computational cost of post-training and make strong reasoning models more accessible to researchers and practitioners with limited hardware resources. This could have positive impacts by reducing energy consumption, lowering deployment costs, and enabling broader participation in research on reasoning-oriented language models.

At the same time, improving the efficiency of reasoning-model distillation may also make capable models easier to train, adapt, and deploy. Such models could be misused for generating misleading content, automating deceptive reasoning, assisting with harmful planning, or increasing the scale of low-cost model deployment without sufficient safety evaluation. Although our experiments focus on mathematical reasoning benchmarks, the same distillation principles may transfer to other domains where stronger reasoning capabilities require careful governance.

Our method does not directly address model safety, factuality, robustness, bias, privacy, or misuse prevention. In particular, SEAD optimizes the efficiency and effectiveness of teacher--student distillation, but it does not guarantee that the student inherits only desirable behaviors from the teacher. If the teacher model exhibits unsafe, biased, or unreliable behavior, the student may also learn such behavior. Therefore, practical deployment of models trained with SEAD should include standard safety evaluations, red-teaming, domain-specific risk assessment, and monitoring for unintended behavior.

The empirical scope of this paper is limited to mathematical reasoning benchmarks and two teacher--student model families. As a result, the broader societal implications of applying SEAD to open-domain assistants, code-generation systems, or high-stakes decision-support systems remain uncertain. We encourage future work to evaluate competence-aware distillation under broader safety, robustness, and alignment criteria before deployment in sensitive settings.

\section{Declaration of LLM Usage}
\label{app:llm-usage}

This paper studies large language models as the primary subject of research. The proposed method uses teacher and student language models during on-policy distillation, as described in the main text.

In addition, the authors used large language model tools to assist with non-substantive writing tasks, including improving clarity, editing grammar, drafting checklist responses, and revising explanatory text. All technical ideas, algorithmic design choices, theoretical claims, proofs, experimental results, tables, and conclusions were produced, checked, and approved by the authors. The authors take full responsibility for the content of the paper, including any text that was edited with LLM assistance.


\end{document}